\documentclass[]{article}
\usepackage{proceed2e}

\newcommand*{\WITHSUPPLEMENT}{}%

\usepackage{times}
\usepackage{url}

\title{Encoding Markov Logic Networks in Possibilistic Logic}

\author{{\bf Ond\v{r}ej Ku\v{z}elka} \\  
School of CS \& Informatics \\  
Cardiff University\\ 
Cardiff, UK \\ 
\And 
{\bf Jesse Davis}  \\ 
Department of Computer Science          \\ 
KU Leuven \\
Leuven, Belgium\\
\And 
{\bf Steven Schockaert}   \\ 
School of CS \& Informatics \\  
Cardiff University\\ 
Cardiff, UK \\  } 

\usepackage{amssymb,amsmath,amsthm}
\usepackage{xspace}
\usepackage{amsthm}
\usepackage{amsmath}
\usepackage{amssymb}
\usepackage{mathrsfs}

\newcommand{\mapent}{\vdash_{\textit{MAP}}}

\newcommand{\cP}{\Theta^k}

\newcommand{\possent}{\vdash_{poss}}

\newtheorem{definition}{Transformation}

\newtheorem{lemma}{Lemma}
\newtheorem{proposition}{Proposition}

\newtheorem{example}{Example}
\newtheorem{corollary}{Corollary}

\usepackage{stmaryrd}
\newcommand{\sem}[1]{\llbracket #1 \rrbracket}

\usepackage{xcolor}

\begin{document}

\maketitle

\begin{abstract}
Markov logic uses weighted formulas to compactly encode a probability distribution over possible worlds. Despite the use of logical formulas, Markov logic networks (MLNs) can be difficult to interpret, due to the often counter-intuitive meaning of their weights. To address this issue, we propose a method to construct a possibilistic logic theory that exactly captures what can be derived from a given MLN using maximum a posteriori (MAP) inference. Unfortunately, the size of this theory is exponential in general. We therefore also propose two methods which can derive compact theories that still capture MAP inference, but only for specific types of evidence. These theories can be used, among others, to make explicit the hidden assumptions underlying an MLN or to explain the predictions it makes.
\end{abstract}

\section{INTRODUCTION}
Markov logic \cite{Richardson2006} and possibilistic logic \cite{DLP} are two popular logics for modelling uncertain beliefs. Both logics share a number of important characteristics. At the syntactic level, formulas correspond to pairs $(\alpha,\lambda)$, consisting of a classical formula $\alpha$ and a certainty weight $\lambda$, while at the semantic level, sets of these formulas induce a mapping from possible worlds to $[0,1]$, encoding the relative plausibility of each possible world.

Despite their close similarities, however, Markov logic and possibilistic logic have been developed in different communities and for different purposes: Markov logic has mainly been studied in a machine learning context whereas possibilistic logic has been studied as a knowledge representation language. This reflects the complementary strengths and weaknesses of these logics. On the one hand, the qualitative nature of possibilistic logic makes it challenging to use for learning; although a few interesting approaches for learning possibilistic logic theories from data have been explored (e.g.\ \cite{Serrurier2007}), their impact on applications to date has been limited. On the other hand, the intuitive meaning of Markov logic theories is often difficult to grasp, which limits the potential of Markov logic for knowledge representation. The main culprit is that the meaning of a theory can often not be understood by looking at the individual formulas in isolation. This issue, among others, has been highlighted in \cite{Thimm:2014b}, where coherence measures are proposed that evaluate to what extent the formulation of a Markov logic theory is misleading. 
\begin{example}\label{ex1}
Consider the following Markov logic formulas:
\begin{align*}
+\infty&: & \textit{antarctic-bird}(X) &\rightarrow \textit{bird}(X)\\
10&: & \textit{bird}(X) &\rightarrow \textit{flies}(X)\\
5&: &  \textit{antarctic-bird}(X) &\rightarrow \neg\textit{flies}(X)
\end{align*}
While the last formula might appear to suggest that antarctic birds cannot fly, in combination with the other two formulas, it merely states that antarctic birds are less likely to fly than birds in general. %In particular, when using maximum a posteriori (MAP) inference, from $\textit{antarctic-bird}(\textit{tweety})$ we will derive that $\textit{flies}(\textit{tweety})$ (see Section \ref{secBackground}).
\end{example}
Possibilistic logic is based on a purely qualitative, comparative model of uncertainty: while a Markov logic theory compactly encodes a probability distribution over the set of possible worlds, a possibilistic logic theory merely encodes a ranking of these possible worlds. Even though a probability distribution offers a much richer uncertainty model, many applications of Markov logic are based on MAP inference, which only relies on the ranking induced by the probability distribution. 

In this paper, we first show how to construct a possibilistic logic theory $\Theta$, given a Markov logic theory $\mathcal{M}$, such that the conclusions that we can infer from $\Theta$ are exactly those conclusions that we can obtain from $\mathcal{M}$ using MAP inference. Our construction can be seen as the syntactic counterpart of the probability-possibility transformation from \cite{probpossTransformation}. In principle, it allows us to combine the best of both worlds, using $\mathcal{M}$ for making predictions while using $\Theta$ for elucidating the knowledge that is captured by $\mathcal{M}$ (e.g.\ to verify that the theory $\mathcal{M}$ is sensible). However, the size of $\Theta$ can be exponential in the size of $\mathcal{M}$, which is unsurprising given that the computational complexity of MAP inference is higher than the complexity of inference in possibilistic logic. To overcome this problem, we begin by studying ground (i.e.\ propositional) theories and propose two novel approaches for transforming a ground MLN into a compact ground possibilistic logic theory that still correctly captures MAP inference, but only for specific types of evidence (e.g.\ sets of at most $k$ literals). Then we lift one of these approaches such that it can transform a first-order MLN into a first-order possibilistic logic theory. Finally, we present several examples that illustrate how the transformation process can be used to help identify unintended consequences of a given MLN, and more generally, to better understand its behaviour.

The remainder of the paper is structured as follows. In the next section, we provide some background on Markov logic and possibilistic logic. In Section \ref{secEncodingGround}, we analyse the relation between MAP inference in ground Markov logic networks and possibilistic logic inference, introducing in particular two methods for deriving compact theories. Section \ref{secEncodingFirstOrder} then discusses how we can exploit the symmetries in the case of an ungrounded Markov logic network, while Section \ref{secExamples} provides some illustrative examples. Finally, we provide an overview of related work in Section \ref{secRelatedWork}.

Due to space limitations, some of the proofs have been omitted from this paper. These proofs can be found in an online appendix.\footnote{\url{http://arxiv.org/abs/1506.01432}}

%****************************************************************
\section{BACKGROUND}\label{secBackground}

\subsection{MARKOV LOGIC}

%Markov logic is a formalism that combines Markov networks with first-order logic.  We presume familiarity with first-order logic. We now introduce Markov networks, Markov logic, and inference for Markov logic.  

%Markov networks~\cite{pietra&al97} are an undirected graphical models that represent a joint probability distribution over a set of propositional variables $\{X_1, \ldots, X_{n}\}$. The graph contains one node for each variable and edges connect variables that are related.  The model has one
%potential function $\phi_k$ for each clique in the graph and it represents the following joint
%distribution: 
%\begin{equation}
%P(X\!=\!x) =\frac{1}{Z} \prod_{k} \phi_k(x_{\{k\}})
%\end{equation}
%\noindent where $x_{\{k\}}$ is the
%state of the $k$th clique (i.e., the state of the variables that
%appear in that clique), and $Z$ is a normalization constant. Markov networks can be represented as a \emph{log-linear} model: $P(X = x) = \frac{1}{Z}\mbox{exp}(\sum_j w_j f_j(\vec{x}))$, where the $f_j$ is a feature and $w_j$ is a weight. A feature may be any real-valued function of the variable assignment.

A Markov logic network (MLN)~\cite{Richardson2006} is a set of pairs $(F, w_F)$, where $F$ is a formula in first-order logic and $w_F$ is a real number, intuitively reflecting a penalty that is applied to possible worlds (i.e.\ logical interpretations) that violate $F$. In examples, we will also use the notation $w_F:\, F$ to denote the formula $(F,w_F)$. An MLN serves as a template for constructing a propositional Markov network. 
%Given a finite set of objects $C$, an MLN induces a Markov network that contains one node for each ground atom and one feature for each ground formula.
In particular, given a set of constants $C$, an MLN $\mathcal{M}$ induces the following probability distribution on possible worlds $\omega$: 
\begin{eqnarray}
p_{\mathcal{M}}(\omega) = \frac{1}{Z}\mbox{exp}\left(\sum_{(F,w_F) \in \mathcal{M}} w_F n_F(\omega)\right),
\label{e:mln}
\end{eqnarray}
\noindent
where $n_F(x)$ is the number of true groundings of $F$ in the possible world $\omega$, and $Z$ %$Z = \sum_{\omega} p(\omega)$ 
is a normalization constant to ensure that $p_\mathcal{M}$ can be interpreted as a probability distribution. Sometimes, formulas $(F,w_F)$ are considered where $w_F = +\infty$, to represent hard constraints. In such cases, we define $p_{\mathcal{M}}(\omega)=0$ for all possible worlds that do not satisfy all of the hard constraints, and only formulas with a real-valued weight are considered in \eqref{e:mln} for the possible worlds that do. Note that a Markov logic network can be seen as a weighted set of propositional formulas, which are obtained by grounding the formulas in $\mathcal{M}$ w.r.t.\ the set of constants $C$ in the usual way. In the particular case that all formulas in $\mathcal{M}$ are already grounded, $\mathcal{M}$ corresponds to a theory in penalty logic \cite{Saint-cyr94penaltylogic}. 

One common inference task in MLNs is full MAP inference. In this setting, given a set of ground literals (the evidence) the goal is to compute the most probable configuration of all unobserved variables (the queries). Two standard approaches for performing MAP inference in MLNs are to employ a strategy based on MaxWalkSAT~\cite{Richardson2006} or to use a cutting plane based strategy~\cite{riedel08,noessner13}. Given a set of ground formulas $E$, we write $\max(\mathcal{M},E)$ for the set of most probable worlds of the MLN that satisfy $E$.
%, i.e.\ $\omega \in \max(\mathcal{M},E)$ if $\omega \models E$ and for each $\omega'$ such that $\omega'\models E$, it holds that $p_{\mathcal{M}}(X=\omega) \geq p_{\mathcal{M}}(X=\omega')$. 
For each $\omega \in \max(\mathcal{M},E)$, $\sum_{(F,w_F) \in \mathcal{M}} w_F n_F(\omega)$ evaluates to the same value, which we will refer to as $\textit{sat}(\mathcal{M},E)$. We define the penalty $\textit{pen}(\mathcal{M},E)$ of $E$ as follows:
$$
\textit{pen}(\mathcal{M},E) = \textit{sat}(\mathcal{M},\emptyset)-\textit{sat}(\mathcal{M},E)
$$
We will sometimes identify possible worlds with the set of literals they make true, writing  $\textit{pen}(\mathcal{M},\omega)$. We will also write $\textit{pen}(\mathcal{M},\alpha)$, with $\alpha$ a ground formula, as a shorthand for $\textit{pen}(\mathcal{M},\{\alpha\})$. We will consider the following inference relation, which has been considered among others in \cite{Saint-cyr94penaltylogic}:
\begin{align}
(\mathcal{M},E) \vdash_{\textit{MAP}} \alpha \quad\text{iff}\quad \forall \omega \in \max(\mathcal{M},E): \omega \models \alpha
\end{align}
with $\mathcal{M}$ an MLN, $\alpha$ a ground formula and $E$ a set of ground formulas. It can be shown that checking $(\mathcal{M},E) \vdash_{\textit{MAP}} \alpha$ for a ground network $\mathcal{M}$ is $\Delta_2^P$-complete\footnote{The complexity class $\Delta_2^P$ contains those decision problems that can be solved in polynomial time on a deterministic Turing machine with access to an NP oracle.} \cite{cayrol1994complexity}.

%----------------
%\subsection{Possibility theory}\label{secBackgroundPossibilityTheory}
\subsection{POSSIBILISTIC LOGIC}
A possibility distribution in a universe $\Omega$ is a mapping $\pi$ from $\Omega$ to $[0,1]$, encoding our knowledge about the possible values that a given variable $X$ can take; throughout this paper, we will assume that all universes are finite. For each $x\in \Omega$, $\pi(x)$ is called the possibility degree of $x$. By convention, in a state of complete ignorance, we have $\pi(x)=1$ for all $x\in \Omega$; conversely, if $X=x_0$ is known, we have $\pi(x_0)=1$ and $\pi(x)=0$ for $x \neq x_0$. Possibility theory \cite{ZadehPossibility,dubois1998possibility} is based on the possibility measure $\Pi$ and dual necessity measure $N$, induced by a possibility distribution $\pi$ as follows ($A\subseteq \Omega$):
\begin{align*}
\Pi(A) &= \max_{a\in A} \pi(a)\\
N(A) &= 1 - \Pi(\Omega\setminus A)
\end{align*}
Intuitively, $\Pi(A)$ is the degree to which available evidence is compatible with the view that $X$ belongs to $A$, whereas $N(A)$ is the degree to which available evidence implies that $X$ belongs to $A$, i.e.\ the degree to which it is certain that $X$ belongs to $A$.

%Possibility degrees are usually interpreted in a purely qualitative way, i.e.\ possibility distributions are merely assumed to encode a total ranking on $\Omega$. However, there exist several interesting ways to give a probabilistic interpretation to possibility degrees. First, a possibility measure $\Pi$ naturally corresponds to a family of probability measures, defined as $\Theta^k(\Pi) = \{ P \,|\, P(A)\leq \Pi(A), \forall A\subseteq \Omega\}$. This view gives rise to the following probability-possibility transformation \cite{probpossTransformation}. Let $p$ be a probability distribution on $\Omega= \{x_1,...,x_n\}$ and assume w.l.o.g.\ that $p(x_i)\geq p(x_{i+1})$. Then $p$ induces a possibility distribution $\pi_p$ defined as $\pi_p(x_1)=1$ and for $i>1$:
%\begin{align*}
%\pi_p(x_i) = 
%\begin{cases}
%\sum_{j=i}^n p(x_j) & \text{if $p(x_{i-1}) > p(x_i)$}\\
%\pi_p(x_{i-1}) & \text{otherwise}
%\end{cases}
%\end{align*}

%----------------

A theory in possibilistic logic \cite{DLP} is a set of formulas of the form $(\alpha,\lambda)$, where $\alpha$ is a propositional formula and $\lambda \in [0,1]$ is a certainty weight. A possibility distribution $\pi$ satisfies $(\alpha,\lambda)$ iff $N(\sem{\alpha}) \geq \lambda$, with $N$ the necessity measure induced by $\pi$ and $\sem{\alpha}$ the set of propositional models of $\alpha$. We say that a possibilistic logic theory $\Theta$ entails $(\alpha,\lambda)$, written $\Theta \models (\alpha,\lambda)$, if every possibility distribution which satisfies all the formulas in $\Theta$ also satisfies $(\alpha,\lambda)$. 
A possibility distribution $\pi_1$ is called less specific than a possibility distribution $\pi_2$ if $\pi_1(\omega)\geq \pi_2(\omega)$ for every $\omega$. It can be shown that the set of models of $\Theta$ always has a least element w.r.t.\ the minimal specificity ordering, which is called the least specific model $\pi^*$ of $\Theta$. It is easy to see that $\Theta \models (\alpha,\lambda)$ iff $\pi^*$ satisfies $(\alpha,\lambda)$.

Even though the semantics of possibilistic logic is defined at the propositional level, we will also use first-order formulas such as $(p(X) \rightarrow q(X,Y), \lambda)$ throughout the paper. As in Markov logic, we will interpret these formulas as abbreviations for a set of propositional formulas, obtained using grounding in the usual way. In particular, we will always assume that first-order formulas are defined w.r.t.\ a finite set of constants.

%Consider the inference relation $\vdash$ defined as follows:
%\begin{align}\label{eqPLrule1}
%&\text{if } (\alpha,\lambda)\in \Theta \text{ then } \Theta \vdash (\alpha,\lambda)\\
%&\text{if } \alpha \equiv \beta \text{ and }  \Theta\vdash (\alpha,\lambda) \text{ then } \Theta\vdash (\beta,\lambda)\\
%&\text{if } \lambda_1\geq \lambda_2\text{ and } \Theta\vdash(\alpha,\lambda_1) \text{ then } \Theta\vdash(\alpha,\lambda_2)\\
%&\text{if } \Theta\vdash (\alpha \vee \beta, \lambda_1) \text{ and } \Theta\vdash(\neg \alpha \vee \gamma, \lambda_2)\label{eqPLrule4}\\
%& \quad\quad\text{ then } \Theta\vdash(\beta \vee \gamma, \min(\lambda_1,\lambda_2))\notag
%\end{align}
%It can be shown that the inference rules \eqref{eqPLrule1}--\eqref{eqPLrule4} are sound and complete, i.e.\ $\Theta \models (\alpha,\lambda)$ iff $\Theta \vdash (\alpha,\lambda)$. 
%Note that $\Theta \vdash (\alpha,\lambda)$ can straightforwardly be checked using a SAT solver by omitting from $\Theta$ all formulas with a certaintly level below $\lambda$ and ignoring the certainty weights of the remaining formulas. It follows that checking $\Theta \vdash (\alpha,\lambda)$ is NP-complete.
%The $\lambda$-cut $\Theta_{\lambda}$ and strict $\lambda$-cut $\Theta_{\underline{\lambda}}$ of a possibilistic logic theory $\Theta$ are defined as follows:
The $\lambda$-cut $\Theta_{\lambda}$ of a possibilistic logic theory $\Theta$ is defined as follows:
\begin{align*}
\Theta_{\lambda} &= \{\alpha \,|\, (\alpha,\mu)\in \Theta, \mu\geq \lambda\}
%\Theta_{\underline{\lambda}} &= \{\alpha \,|\, (\alpha,\mu)\in \Theta, \mu> \lambda\}
\end{align*}
It can be shown that $\Theta \models (\alpha,\lambda)$ iff $\Theta_{\lambda} \models \alpha$, which means that inference in possibilistic logic can straightforwardly be implemented using a SAT solver. 

An inconsistency-tolerant inference relation $\possent$ for possibilistic logic can be defined as follows:
\begin{align*}
\Theta \possent \alpha \quad\text{iff}\quad \Theta_{\textit{con}(\Theta)} \models \alpha
\end{align*}
where the consistency level $\textit{con}(\Theta)$ of $\Theta$ is the lowest certainty level $\lambda$ for which $\Theta_{\lambda}$ is satisfiable (among the certainty levels that occur in $\Theta$). Note that all formulas with a certainty level below $\textit{con}(\Theta)$ are ignored, even if they are unrelated to any inconsistency in $\Theta$. This observation is known as the drowning effect.

We will write $(\Theta,E)\possent \alpha$, with $E$ a set of propositional formulas, as an abbreviation for $\Theta\cup \{(e,1)\,|\,e\in E\} \possent \alpha$. Despite its conceptual simplicity, $\possent$ has many desirable properties. Among others, it is closely related to AGM belief revision \cite{273026} and default reasoning \cite{benferhat1997nonmonotonic}. 
It can be shown that checking $\Theta \possent (\alpha,\lambda)$ is a $\Theta_2^P$ complete problem\footnote{The complexity class $\Theta_2^P$ contains those decision problems that can be solved in polynomial time on a deterministic Turing machine, by making at most a logaritmic number of calls to an NP oracle.} \cite{La2001.1}. 
In this paper, $\possent$ will allow us to capture the non-monotonicity of MAP inference.
\begin{example}
Let $\Theta$ consist of the following formulas:
\begin{align*}
&(\textit{penguin}(X) \rightarrow \textit{bird}(X), 1)\\
&(\textit{penguin}(X) \rightarrow \neg \textit{flies}(X), 1)\\
&(\textit{bird}(X) \rightarrow \textit{flies}(X), 0.5)
\end{align*}
Then we find:
\begin{align*}
&(\Theta, \{\textit{bird}(\textit{tweety})\}) \possent\textit{flies}(\textit{tweety})\\
&(\Theta, \{\textit{bird}(\textit{tweety}),\textit{penguin}(\textit{tweety})\}) \possent \neg\textit{flies}(\textit{tweety})
\end{align*}
\end{example}
In general, $\possent$ allows us to model rules with exceptions, by ensuring that rules about specific contexts have a higher certainty weight than rules about general contexts. 
%****************************************************************
\section{ENCODING GROUND NETWORKS}\label{secEncodingGround}
Throughout this section, we will assume that $\mathcal{M}$ is a ground MLN in which all the weights are strictly positive. This can always be guaranteed for ground MLNs by replacing formulas $(\alpha,\lambda)$ with $\lambda < 0$ by $(\neg \alpha,-\lambda)$, and by discarding any formula whose weight is 0.
For a subset $X \subseteq \mathcal{M}$, we write $X^*$ for the set of corresponding classical formulas, e.g.\ for $X=\{(F_1,w_1),...,(F_n,w_n)\}$ we have $X^*=\{F_1,...,F_n\}$. In particular, $\mathcal{M}^*$ are the classical formulas appearing in the MLN $\mathcal{M}$.  

The following transformation constructs a possibilistic logic theory that is in some sense equivalent to a given MLN. It is inspired by the probability-possiblity transformation from \cite{probpossTransformation}. 
\begin{definition}\label{eqStandardEncoding}
We define the possibilistic logic theory $\Theta_{\mathcal{M}}$ corresponding to an MLN $\mathcal{M}$ as follows:
%\begin{align}\label{eqStandardEncoding}
%\Theta_{\mathcal{M}} = &\{( \bigvee X^*, \phi(\neg \bigvee X^*)) \,|\, X\subseteq \mathcal{M},\\
%&\quad\quad\quad\quad\quad\quad\quad\quad\phi(\neg \bigvee X^*)>\phi(\top)\}\notag
%\end{align}
\begin{align}%\label{eqStandardEncoding}
\{( \bigvee X^*, \phi(\neg \bigvee X^*)) \,|\, X\subseteq \mathcal{M},\phi(\neg \bigvee X^*) > 0\}
\end{align}
where for a propositional formula $\alpha$:
\begin{align*}
\phi(\alpha) = 
\begin{cases}
\frac{K + \textit{pen}(\mathcal{M},\alpha)}{L} & \text{if $\alpha$ satisfies the hard constraints}\\
1 & \text{otherwise}
\end{cases}
\end{align*}
and the constants $K$ and $L$ are chosen such that $0=\phi(\top)\leq \phi(\alpha) < 1$ for every $\alpha$ that satisfies the hard constraints (i.e.\ the formulas with weight $+\infty$).
\end{definition}
In the following we will use the notations $\phi(\omega)$ for a possible world $\omega$ and $\phi(E)$ for a set of formulas $E$, defined entirely analogously. Throughout the paper we will also write ($-K < x < L-K$):
$$
\lambda_x =  \frac{K + x}{L}
$$
The correctness of Transformation \ref{eqStandardEncoding} follows from the next proposition, which is easy to show.
\begin{proposition}\label{prop1}
Let $\mathcal{M}$ be a ground MLN and $\Theta_{\mathcal{M}}$ the corresponding possibilistic logic theory. Let $\pi$ be the least specific model of $\Theta_\mathcal{M}$. It holds that:
$$
\pi(\omega) = 1 - \phi(\omega)
$$
\end{proposition}

\begin{corollary}\label{propCorrespondencePenalties}
Let $\mathcal{M}$ be a ground MLN and $\Theta_{\mathcal{M}}$ the corresponding possibilistic logic theory. It holds that for $\lambda <1$:
$$
\Theta_\mathcal{M} \models (\alpha,\lambda) \quad\text{iff}\quad \textit{pen}(\mathcal{M},\neg\alpha)\geq\lambda L - K
$$
and
$$
\Theta_\mathcal{M} \models (\alpha,1) \quad\text{iff}\quad \textit{pen}(\mathcal{M},\neg\alpha)=+\infty
$$
\end{corollary}
\begin{corollary}\label{propCorrespondence}
Let $\mathcal{M}$ be a ground MLN and $\Theta_{\mathcal{M}}$ the corresponding possibilistic logic theory. For $p_{\mathcal{M}}$ the probability distribution induced by $\mathcal{M}$ and $\pi$ the least specific model of $\Theta_\mathcal{M}$, it holds that
$$
p_{\mathcal{M}}(\omega_1) > p_{\mathcal{M}}(\omega_2) \quad \text{iff}\quad \pi(\omega_1) > \pi (\omega_2)
$$
for all possible worlds $\omega_1$ and $\omega_2$. In particular, it follows that for every propositional formula $\alpha$ and every set of propositional formulas $E$:
\begin{align}\label{eqInferenceEquivalence}
(\mathcal{M},E) \vdash_{\textit{MAP}} \alpha \quad\text{iff}\quad (\Theta,E) \possent \alpha
\end{align}
\end{corollary}
\begin{example}\label{exTransformationBasic}
Consider the MLN $\mathcal{M}$ containing the following formulas:
\begin{align*}
5:&\,\, a \rightarrow x&
5:&\,\, a \rightarrow y&
10:&\,\, a\wedge b \rightarrow \neg y
\end{align*}
Then $\Theta_{\mathcal{M}}$ contains the following formulas:
\begin{align*}
\lambda_5: &\,\, a \rightarrow x &
\lambda_5: &\,\, a \rightarrow y\\
\lambda_{10}: &\,\, a\wedge b \rightarrow \neg y&
\lambda_{10}: &\,\, a \rightarrow x\vee y\\
\lambda_{15}: &\,\, a\wedge b \rightarrow x\vee \neg y
\end{align*}
It can be verified that:
\begin{align*}
&(\Theta_{\mathcal{M}},\{a\}) \possent x\wedge y&
&(\Theta_{\mathcal{M}},\{a,b\}) \possent x\wedge \neg y
\end{align*}
\end{example}
An important drawback of the transformation to possibilistic logic is that the number of formulas in $\Theta_{\mathcal{M}}$ is exponential in $|\mathcal{M}|$. This makes the transformation inefficient, and moreover limits the interpretability of the possibilistic logic theory. In general, the exponential size of $\Theta_{\mathcal{M}}$ cannot be avoided if we want \eqref{eqInferenceEquivalence} to hold for any $E$ and $\alpha$. However, more compact theories can be found if we focus on specific types of evidence. Sections \ref{secSelectiveDrowning} and \ref{section:defaultEncoding} introduce two practical methods to accomplish this.

%------------------------------------------------------------
\subsection{SELECTIVELY AVOIDING DROWNING}\label{secSelectiveDrowning}
In many applications, we are only interested in particular types of evidence sets $E$. For example, we may only be interested in evidence sets that contain at most $k$ literals, or in evidence sets that only contain positive literals. In such cases, we can often derive a more compact possibilistic logic theory $\Theta^{\mathcal{E}}$ as follows. Let $\mathcal{E}$ be the set of  evidence sets that we wish to consider, where each $E \in \mathcal{E}$ is a set of ground formulas.  Given $E\in \mathcal{E}$ we write $\mathcal{S}_E$ for the set of all minimal subsets $\{F_1,...,F_l\}$ of $\mathcal{M}^*_E = \{F \,|\, F\in \mathcal{M}^*, \textit{pen}(\mathcal{M}, \neg F) < \textit{pen}(\mathcal{M}, E)\}$ s.t.
\begin{align}\label{eqFormulasDrowningX}
\textit{pen}(\mathcal{M}, \bigwedge E \wedge \neg F_1 \wedge ... \wedge \neg F_l) > \textit{pen}(\mathcal{M}, E)
\end{align}
The following transformation constructs a possibilistic logic theory that correctly captures MAP inference for evidence sets in $\mathcal{E}$. The basic intuition is that we want to weaken the formulas in $\mathcal{M}^*$ just enough to ensure that the resulting certainty level prevents them from drowning when the evidence $E$ becomes available.
\begin{definition}\label{transformationSelectedEvidence}
Given a ground MLN $\mathcal{M}$ and a set of evidence sets $\mathcal{E}$, we define the possibilistic logic theory $\Theta^{\mathcal{E}}_{\mathcal{M}}$ as follows:
\begin{align}
&\{ (F_1, \phi(\neg F_1)) \,|\, F_1 \in \mathcal{M}^*\}\label{eqFormulasDrowningA0}\\
& \cup \{(\neg \bigwedge E\vee \bigvee Z , \phi(\bigwedge E \wedge \neg \bigwedge Z))\,|\, Z \in \mathcal{S}_E, \label{eqFormulasDrowningA}\\
& \quad E\in \mathcal{E}\} \cup \{(\neg \bigwedge E, \phi(\bigwedge E))\,|\, E\in\mathcal{E}\} \label{eqFormulasDrowningB}
\end{align}
\end{definition}
If $\mathcal{M}$ is clear from the context, we will omit the subscript in $\Theta^{\mathcal{E}}_{\mathcal{M}}$.
The formulas in \eqref{eqFormulasDrowningA0} are the direct counterpart of the MLN.
Intuitively, there are two reasons why these formulas are not sufficient. First, due to the drowning effect, formulas $F$ such that $\textit{pen}(\mathcal{M}, \neg F) < \textit{pen}(\mathcal{M}, E)$ will be ignored under the evidence $E$. In such cases we should look at minimal ways to weaken these formulas such that the certainty level of the resulting formula is sufficient to avoid drowning under the evidence $E$. This is accomplished by adding the formulas in \eqref{eqFormulasDrowningA}. Second, as $\Theta^{\mathcal{E}}$ contains less information than $\Theta_{\mathcal{M}}$, we need to ensure that the consistency level for $\Theta^{\mathcal{E}}$ is never lower than the consistency level for $\Theta_{\mathcal{M}}$, given an evidence set $E \in \mathcal{E}$. To this end, $\Theta^{\mathcal{E}}$ includes the formulas in \eqref{eqFormulasDrowningB}. The following example illustrates why these formulas are needed.
  
\begin{example}\label{exSelectiveDrowningA}
Consider the following MLN $\mathcal{M}$:
\begin{align*}
3:&\,\, u &
2:&\,\, a &
10:&\,\, (a\vee b) \wedge (u\vee v) \rightarrow \neg x  \\
2:&\,\, b &
1:&\,\, v
\end{align*}
and let $\mathcal{E} = \{ \{x\} \}$, i.e.\ the only evidence set in which we are interested is $\{x\}$. It holds that
\begin{align}\label{eqSEexample4}
\mathcal{S}_E = \{\{a, u\}, \{b, u\}, \{a, v\}, \{b, v\}\}
\end{align}
and $\Theta^{\mathcal{E}} = \Theta \cup \Psi \cup \Gamma$, where:
\begin{align*}
\Theta = \{& (u,\lambda_3), (a,\lambda_2), ((a\vee b) \wedge (u\vee v) \rightarrow \neg x, \lambda_{10}),\\
& (b,\lambda_2), (v,\lambda_1) \}\\
\Psi = \{& (a\vee u \vee \neg x, \lambda_6), (b\vee u \vee \neg x, \lambda_6),\\
& (a\vee v \vee \neg x, \lambda_5),(b\vee v \vee \neg x, \lambda_5)\\
\Gamma = \{& (\neg x, \lambda_4)\}
\end{align*}
It is easy to verify that $(\Theta\cup \Psi,\{x\})\possent u$ whereas $(\mathcal{M},\{x\}) \not \vdash_{\textit{MAP}} u$ and $(\Theta\cup \Psi \cup \Gamma,\{x\})\not \possent u$.
\end{example}
We now prove the correctness of Transformation \ref{transformationSelectedEvidence}.
\begin{proposition}
For any formula $\alpha$ and any evidence set $E\in\mathcal{E}$, it holds that $(\Theta_{\mathcal{M}},E)\possent \alpha$ iff $(\Theta^{\mathcal{E}},E)\possent \alpha$.
\end{proposition}
\begin{proof}
Let us introduce the following notation:
\begin{align*}
\lambda_E &=\textit{con}(\Theta_{\mathcal{M}} \cup \{(e,1)\,|\, e\in E\})\\
\lambda_E^{\mathcal{E}} &=\textit{con}(\Theta^{\mathcal{E}} \cup \{(e,1)\,|\, e\in E\})\\
A&=(\Theta_{\mathcal{M}} \cup \{(e,1)\,|\, e\in E\})_{\lambda_E}\\
A^E&=(\Theta^{\mathcal{E}} \cup \{(e,1)\,|\, e\in E\})_{\lambda^{\mathcal{E}}_E}
\end{align*}
We need to show that $A$ is equivalent to $A^E$, for any $E\in\mathcal{E}$.

By Corollary \ref{propCorrespondencePenalties}, we know that every formula $(\alpha,\lambda)$ in $\Theta^{\mathcal{E}}_{\mathcal{M}}$ is entailed by  $\Theta_{\mathcal{M}}$, hence $\lambda_E^{\mathcal{E}}\leq \lambda_E$. Since $\lambda_E$ is the smallest certainty level from $\Theta_{\mathcal{M}}$ which is strictly higher than $\phi(E)$, it follows that $A^{E}$ contains every formula which appears in $\Theta^{\mathcal{E}}$ with a weight that is strictly higher than $\phi(E)$. Moreover, since $\Theta^{\mathcal{E}}$ by construction contains $(\neg \bigwedge E, \phi(\bigwedge E))$, we find that $A^E$ can only contain such formulas:
\begin{align}\label{eqLemmaConsistencyLevelEquivalence}
A^E&= E \cup \{\alpha \,|\, (\alpha,\lambda)\in \Theta^{\mathcal{E}}, \lambda > \phi(E)\}
\end{align}
It follows that $A \models A^{E}$.

Let $G_1 \vee ... \vee G_s$ be a formula from $A$. From Corollary \ref{propCorrespondencePenalties} we  know that:
$$
\textit{pen}(\mathcal{M},\neg G_1 \wedge ... \wedge \neg G_s) > \textit{pen}(\mathcal{M}, E)
$$
and a fortiori
$$
\textit{pen}(\mathcal{M}, E \wedge \neg G_1 \wedge ... \wedge \neg G_s) > \textit{pen}(\mathcal{M}, E)
$$
This means that for any formula $G_1 \vee ... \vee G_s$ in $A$, either $\textit{pen}(\mathcal{M},\neg G_i) > \textit{pen}(\mathcal{M}, E)$ for some $i$  or  $\mathcal{S}_E$ contains a subset $\{H_1,..., H_r\}$ of 
$\{G_1,...,G_s\}$. Then $\Theta^{\mathcal{E}}$ contains either $G_i$ or the formula $\neg E \vee H_1 ... \vee H_r$ with a weight which is strictly higher than $\phi(E)$ and thus either $G_i$ or $\neg E \vee H_1 ... \vee H_r$ belongs to $A^{E}$. In both cases we find $A^E \models G_1 \vee ... \vee G_s$. We conclude $A^{E} \models A$.
\end{proof}
An alternative, which would make the approach in this section closer to the standard encoding in \eqref{eqStandardEncoding}, is to define $\mathcal{S}'_E$ as the set of minimal subsets $\{F_1,...,F_l\}$ of $\mathcal{M}^*_E$ such that
\begin{align}\label{eqFormulasDrowningBisX}
\textit{pen}(\mathcal{M},\neg F_1 \wedge ... \wedge \neg F_l) > \textit{pen}(\mathcal{M}, E)
\end{align}
and then replace the formulas in \eqref{eqFormulasDrowningA} by
\begin{align}\label{eqFormulasDrowningBisA}
\{(\bigvee Z , \phi( \neg \bigwedge Z))\,|\, Z \in \mathcal{S}'_E\}
\end{align}
The advantage of \eqref{eqFormulasDrowningA}, however, is that we can expect many of the sets in $\mathcal{S}_E$ to be singletons. To see why this is the case, first note that for each world $\omega$ in $\max(\mathcal{M},E)$, the set of formulas $\mathcal{Y} \subseteq \mathcal{M}^*$ satisfied by $\omega$ is such that $\textit{pen}(\mathcal{M},\neg \bigvee(\mathcal{M}^*\setminus \mathcal{Y}))$ is minimal among all sets $\mathcal{Y}' \subseteq \mathcal{M}^*$ for which $E \wedge \bigwedge \mathcal{Y}'$ is consistent. Let us write $\textit{Cons}_E(\mathcal{M})$ for the set of all these maximally consistent subsets of $\mathcal{M}^*$. Note that  $\max(\mathcal{M},E)=\sem{\bigvee \{ \bigwedge \mathcal{Y} \,|\, \mathcal{Y}\in \textit{Cons}_E(\mathcal{M})\}}$.

\begin{lemma}\label{LemmaSelectiveMAPStates}
For a set of formulas $\{F_1,...,F_l\} \subseteq \mathcal{M}^*$ it holds that $\textit{pen}(\mathcal{M}, E \wedge \neg F_1 \wedge ... \wedge \neg F_l) > \textit{pen}(\mathcal{M}, E)$ iff $\{F_1,...,F_l\} \cap \mathcal{Y} \neq \emptyset$ for every $\mathcal{Y}$ in $\textit{Cons}_E(\mathcal{M})$.
\end{lemma}
\begin{corollary}\label{corME}
Let $\textit{Cons}_E(\mathcal{M}) = \{\mathcal{Y}_1,...,\mathcal{Y}_s\}$. It holds that $\mathcal{S}_E$ consists of the subset-minimal elements of $\{ \{y_1,...,y_s\} \,|\, y_1\in \mathcal{Y}_1 \cap \mathcal{M}^*_E,...,y_s \in \mathcal{Y}_s\cap \mathcal{M}^*_E \}$.
\end{corollary}
\begin{example}
Consider again the MLN $\mathcal{M}$ from Example \ref{exSelectiveDrowningA} and let $E=\{x\}$. It holds that $\textit{Cons}_E(\mathcal{M}) = \{\mathcal{Y}_1,\mathcal{Y}_2\}$, where
\begin{align*}
\mathcal{Y}_1 &= \{(a\vee b) \wedge (u\vee v) \rightarrow \neg x, a,b\}\\
\mathcal{Y}_2 &= \{(a\vee b) \wedge (u\vee v) \rightarrow \neg x, u,v\}\\
\mathcal{M}^*_E &= \{a,b,u,v\}
\end{align*}
From Corollary \ref{corME},it follows that $\mathcal{S}_E$ is given by \eqref{eqSEexample4}.
\end{example}
In practice, $\textit{Cons}_E(\mathcal{M})$ will often contain a single element, in which case all the elements of $\mathcal{S}_E$ will be singletons.

%------------------------------------------------------------
\subsection{MAP INFERENCE AS DEFAULT REASONING}\label{section:defaultEncoding}
A large number of approaches has been proposed for reasoning with a set of default rules of the form ``if $\alpha$ then typically $\beta$'' \cite{KLM,pearl1990system,geffner1992conditional}. At the core, each of the proposed semantics corresponds to the intuition that a set of default rules imposes a preference order on possible worlds, where ``if $\alpha$ then $\beta$'' means that $\beta$ is true in the most preferred models of $\alpha$. The approaches from \cite{KLM} and  \cite{pearl1990system} can be elegantly captured in possibilistic logic  \cite{benferhat1997nonmonotonic}, by interpreting the default rule as the constraint $\Pi(\alpha \wedge \beta) > \Pi(\alpha \wedge \neg \beta)$. In Markov logic, the same constraint on the ordering of possible worlds can be expressed by imposing the constraint $(\mathcal{M},\alpha) \mapent \beta$. In other words, we can view the MAP consequences of an MLN as a set of default rules, and encode these default rules in possibilistic logic. The following transformation is based on this idea.
\begin{definition}\label{transformationDefaults}
Given a ground MLN $\mathcal{M}$ and a positive integer $k$, we construct a possibilistic logic theory $\Theta^k_{\mathcal{M}}$ as follows:
\begin{itemize}
\item For each hard rule $F$ from $\mathcal{M}$, add $(F,1)$ to $\Theta^k_{\mathcal{M}}$.
\item For each set of literals $E$ such that $0 \leq |E| \leq k$, let $X = \{x \,|\, (\mathcal{M},E) \mapent x\}$ be the set of literals that are true in all the most plausible models of $E$. 
Unless there is a literal $y \in E$ such that $\bigwedge (E \setminus \{y\}) \mapent y$, add 
$$\left(\bigwedge E \rightarrow \bigwedge X, \lambda_E \right)$$
 to $\Theta^k_{\mathcal{M}}$, where $\lambda_E =  \phi(\bigwedge E)$. If $\textit{pen}(\mathcal{M},E)>\textit{pen}(\mathcal{M},\emptyset)$, add also
\begin{align}\label{eqBlockingRuleDefaults}
(\neg (\bigwedge E \wedge \bigwedge X), \lambda_E')
\end{align}
where $\lambda_E'$ is the certainty level just below $\lambda_E$ in $\Theta^k_{\mathcal{M}}$, i.e.\ $\lambda_{E'}=\max\{\lambda_F \,|\, \lambda_F<\lambda_E, |F|\leq k\}$.
\end{itemize}
\end{definition}
If $\mathcal{M}$ is clear from the context, we will omit the subscript in $\Theta^k_{\mathcal{M}}$. 
The possibilistic encoding of default rules used in Transformation \ref{transformationDefaults} is similar in spirit to the method from \cite{benferhat1997nonmonotonic}, which is based on the Z-ranking from \cite{pearl1990system}. However, because $p_{\mathcal{M}}$ already provides us with a model of the default rules, we can directly encode default rules in possibilistic logic, without having to rely on the Z-ranking. Also note that although the method is described in terms of an MLN, it can be used for encoding any ranking on possible worlds (assuming a finite set of atoms).

As illustrated in the following example, \eqref{eqBlockingRuleDefaults} is needed to avoid deriving too much, serving a similar purpose to \eqref{eqFormulasDrowningB} in the approach from Section \ref{secSelectiveDrowning}. 
\begin{example}
Consider the following MLN $\mathcal{M}$:
\begin{align*}
2:&\,\, \neg a \vee b &
2:&\,\, a \vee b &
1:&\,\, a \vee \neg b
\end{align*}
Then $\Theta^1 = \Theta \cup \Psi$, where
\begin{align*}
\Theta &= \{(\top \rightarrow a\wedge b,\lambda_0), (\neg a \rightarrow b,\lambda_1), (\neg b \rightarrow \top,\lambda_2)\}\\
\Psi &= \{(b,\lambda_1),(a \vee \neg b,\lambda_0)\} 
\end{align*}
We find $(\Theta,\{\neg b\})\possent a$ while $(\mathcal{M},\{\neg b\})\not\mapent a$. Accordingly, we have $(\Theta\cup \Psi,\{\neg b\})\not\possent a$.
\end{example}
Transformations 2 and 3 have complementary strengths. For example, Transformation 2 may lead to more compact theories for relatively simple MLNs, e.g.\ if for most of the considered evidence sets, there is a unique set of formulas from the MLN that characterizes the most probable models of the evidence (cf.\ Lemma \ref{LemmaSelectiveMAPStates}). On the other hand, Transformation 3 may lead to substantially more compact theories in cases where the number of formulas is large relative to the number of atoms. 
%Next, we will show that the resulting MLN satisfies the desirable properties stated in the next proposition. The evidence sets considered here will be supposed to be sets of literals, i.e. we will not assume that evidence could contain clauses. {\color{red} Maybe we should say this somewhere else.}

We now show the correctness of Transformation \ref{transformationDefaults}.
\begin{proposition}\label{prop:defaultbased}
Let $\mathcal{M}$ be an MLN, $k$ a positive integer and $\Theta^k$ the proposed possibilistic logic encoding of $\mathcal{M}$. Furthermore, let  $E$ and $C$ be sets of literals such that $|E|+|C| \leq k+1$. It holds that $(\mathcal{M},E) \mapent \bigvee C$ if and only if $(\Theta^k,E) \possent \bigvee C$.
%Then the following holds:
%\begin{enumerate}
%\item[(i)] For sets of literals $E$ and and $C$ such that $|E|+|C| \leq k+1$, it holds that $(\mathcal{M},E) \mapent \bigvee C$ if and only if $(\Theta^k,E) \possent \bigvee C$.
%\item[(ii)] If no limit on the evidence set size is given (i.e. if $k$ is greater than the maximum possible evidence set size) then $(\mathcal{M},E) \mapent \bigvee C$ if and only if $(\Theta^k,E) \possent \bigvee C$ for any evidence set $E$ and any clause $\bigvee C$.
%\end{enumerate}
\end{proposition}

Before we prove Proposition \ref{prop:defaultbased}, we present a number of lemmas.
%We just note here that (ii) follows easily from (i), so we will only prove (i). 
In the lemmas and proofs below, $\mathcal{M}$ will always be an MLN, $\Theta^k$ will be the corresponding possibilistic logic theory and $k$ will be the maximum size of the evidence sets considered in the translation.% and $K$, $L$ will be the normalization constants used.

\begin{lemma}\label{lemma:firstlemma}
If $E$ is a set of literals, $|E| \leq k$, $\lambda = \phi(E)$ and $(\mathcal{M},E) \mapent x$ then 
$$\left\{ \left( \bigwedge E' \rightarrow \bigwedge X \right) \in \Theta^k_{\lambda} \mbox{ s.t. } |E'| \leq |E| \right\} \vdash  \bigwedge E \rightarrow x$$
\end{lemma}

\begin{lemma}\label{lemma:lemma2}
If $\phi( \omega) \leq \lambda$ then $\omega$ is a model of $\cP_\lambda$.
\end{lemma}
\begin{proof}
If there were a formula $F = \left(\bigwedge E \right) \rightarrow \left(\bigwedge X \right)$ in $\cP_\lambda$ that was not satisfied by $\omega$, then its body would have to be true in $\omega$ but then necessarily 
%$$\lambda \leq \frac{\textit{pen}(\mathcal{M}, E)+K}{L}  \leq  \frac{\textit{pen}(\mathcal{M}, \omega)+K}{L} \leq \lambda.$$
$$\lambda \leq \phi( E)  \leq  \phi( \omega) \leq \lambda.$$
The first inequality follows from the fact that, by the construction of $\cP$, if the certainty weight of $F$ is at least $\lambda$ then it must be the case that $\phi( E) \geq \lambda$. The second inequality follows from the fact that $\omega$ was assumed to be a model of $\bigwedge E$. It follows that:
$$\textit{pen}(\mathcal{M}, \omega) = \textit{pen}(\mathcal{M},E).$$
\noindent However, this would mean that $\omega$ is also a most probable world of $(\mathcal{M},E)$, but then $\omega\models F$ by construction of $\cP$.
%$F$ could not contain any literal inconsistent with $\omega$ in its head (by the construction of $\cP$, any formula $\left( \bigwedge E \right) \rightarrow \left(\bigwedge X\right)$ is true in any most probable world of $(\mathcal{M},E)$). 

If there were an unsatisfied formula $F = \neg \left(\bigwedge E \wedge \bigwedge X\right)$ in $\cP_\lambda$ then by construction we would have $\phi(E \cup X) > \lambda$. However, from $\omega\models \bigwedge E \wedge \bigwedge X$ we find $\phi(E \cup X)\leq \phi(\omega)\leq \lambda$, a contradiction.

Since all formulas in $\Theta^k$ are of the two considered types, it follows that all formulas from $\cP$ whose certainty weight is at least $\lambda$ must be satisfied in $\omega$.
\end{proof}

\begin{lemma}\label{lemma:moving}
If $(\mathcal{M}, E) \mapent (y_1 \vee \dots \vee y_m)$ then 
\begin{enumerate}
\item[(i)] for any $i$, either $(\mathcal{M}, E \cup \{ \neg y_i \}) \mapent (y_1 \vee \dots \vee y_{i-1} \vee y_{i+1} \vee \dots \vee y_m)$ or $(\mathcal{M}, E) \mapent y_i$,
\item[(ii)] there exist a $j$ and a set $\{ y_1', \dots, y_{m'}' \}  \subseteq \{ y_1, \dots, y_m \} \setminus \{ y_j \}$ such that $(\mathcal{M}, E \cup \{ \neg y_1', \dots, \neg y_{m'}' \}) \mapent y_j$.
\end{enumerate}
\end{lemma}

\begin{lemma}\label{lemma:clausalconsequence}
If $|C|+|E| \leq k+1$, and $\lambda = \phi(E)$ then $\cP_\lambda \cup E \vdash \bigvee C$ if and only if $(\mathcal{M},E) \mapent \bigvee C$.
\end{lemma}
We now turn to the proof of Proposition \ref{prop:defaultbased}.
\begin{proof}[Proof of Proposition \ref{prop:defaultbased}]
Let $E$ be an evidence set such that $|E| \leq k$ and let $\lambda = \phi(E)$. Given Lemma \ref{lemma:clausalconsequence}, it is sufficient to show that $\textit{con}(\Theta^k,E)=\lambda$. It follows from Lemma \ref{lemma:lemma2} that $\textit{con}(\Theta^k,E) \leq \lambda$. Let $X = \{x \,|\, (\mathcal{M},E) \mapent x \}$ be the set of literals which can be derived from $(\mathcal{M},E)$ using MAP inference. 
%Then, since $|E| \leq k$, it follows from Lemma \ref{lemma:clausalconsequence} that $(\Theta^k_\lambda,E) \vdash \bigwedge X$. 
By construction, $\Theta^k$ contains a formula $\neg (\bigwedge E \wedge \bigwedge X)$ with a certainty weight which is just below $\lambda$. Specifically, for $\lambda'<\lambda$ we either have $\Theta^k_\lambda = \Theta^k_{\lambda'}$ or $\Theta^k_{\lambda'} \models \neg \bigwedge E$, from which we find $\textit{con}(\Theta^k,E) = \lambda$.
\end{proof}

%\paragraph{Filtering redundant rules} 
It is of interest to remove any formulas in $\Theta^k$ that are redundant, among others because this is likely to make the theory easier to interpret. Although we can use possibilistic logic inference to identify redundant formulas, in some cases we can avoid adding the redundant formulas altogether. For example, in the transformation procedure, we do not add any rules for $E$ if it holds that $E\setminus \{y\} \mapent y$ for some $y\in E$. This pruning rule is the counterpart of the cautious monotonicity property, which is well-known in the context of default reasoning \cite{KLM}. Any ranking on possible worlds also satisfies the stronger rational monotonicity property, which translated to our setting states that when $(\mathcal{M},E\setminus \{y\}) \mapent x$ and $(\mathcal{M},E\setminus \{y\}) \not\mapent \neg y$ it holds that $(\mathcal{M},E) \mapent x$. Accordingly, when processing the evidence set $E$ in the transformation procedure, instead of $(\bigwedge E \rightarrow \bigwedge X,\lambda_E)$ it is sufficient to add the following rule:
$$
(\bigwedge E \rightarrow \bigwedge (X \setminus X_0) ,\lambda_E)
$$
where
$$
X_0 = \{x \,|\, E\setminus \{y\} \mapent x \textit{ and } E\setminus\{y\} \not\mapent \neg y \}
$$
The correctness of this pruning step follows from the following proposition.

%Some of the rules included in the constructed possibilistic logic theory may be redundant. A rather expensive form of removing these rules is postprocessing of the possibilistic logic theory in which rules which can be derived from other rules in the theory with higher or equal necessities are iteratively removed. Since this postprocessing requires running inference in the possibilistic logic theory, it is beneficial to prune the rules as much as possible already during the transformation process. The transformation process as described above already performs a limited form of pruning of rules. The next proposition shows that even more rules can be pruned using a property which is closely related to rational monotonicity from default-based reasoning.

\begin{proposition}\label{prop:rationalmapmonotonicity}
Let $x$ and $y$ be literals. If $|E| < k$, $(\mathcal{M},E) \mapent x$ and $(\mathcal{M},E) \not\mapent \neg y$ then:
\begin{align*}
\Theta^k \setminus \{F\} \models  \big( \bigwedge E \wedge y \rightarrow x, \lambda_{E\cup \{ y \}} \big)
\end{align*}
where $F$ is the formula in $\Theta^k$ corresponding to the evidence set $E\cup \{y\}$, i.e.:
$$F = \bigwedge (E\cup \{y\}) \rightarrow \bigwedge \{x \,|\, (\mathcal{M},E \cup \{y\}) \mapent x\}$$
\end{proposition}
\begin{proof}
If $(\mathcal{M},E) \mapent x$ and $(\mathcal{M},E) \not\mapent \neg y$ then $(\mathcal{M},E \cup \{ y \}) \mapent x$ and 
$\textit{pen}(\mathcal{M},E) = \textit{pen}(\mathcal{M},E \cup \{ y \}) = \textit{pen}(\mathcal{M},E \cup \{ x,y \}).$ Therefore using Lemma \ref{lemma:firstlemma}, we find that $\Theta^k_{\lambda_{E\cup \{ y \}}} \vdash \bigwedge E \rightarrow x$. From Lemma \ref{lemma:firstlemma}, it furthermore follows that $\bigwedge E \rightarrow x$ can be derived from rules with antecedents of length at most $|E|$. In particular, we find that $\bigwedge E \rightarrow x$ can be derived without using the formula $\bigwedge E \wedge y \rightarrow \bigwedge X$.
\end{proof}
Finally, note that formulas of the form \eqref{eqBlockingRuleDefaults} can be omitted when $\lambda'_E = \lambda_{(E\setminus \{y\})}$ for some $y\in E$. Indeed, in such a case we find from $\textit{pen}(\mathcal{M},E\setminus \{y\})< \textit{pen}(\mathcal{M},E)$ that $(\mathcal{M},E\setminus \{y\})\mapent \neg y$, hence $\eqref{eqBlockingRuleDefaults}$ will be entailed by a formula of the form $\big(\bigwedge (E\setminus \{y\})\rightarrow \bigwedge X,\lambda_{E\setminus \{y\}}\big)$ in $\Theta^k$.

%****************************************************************

\section{ENCODING NON-GROUND NETWORKS}\label{secEncodingFirstOrder}

We now provide the counterpart to the construction from Section \ref{section:defaultEncoding} for non-ground MLNs.
The first-order nature of MLNs often leads to distributions with many symmetries which can be exploited by lifted inference methods \cite{Poole2003}. We can similarly exploit these symmetries for constructing more compact possibilistic logic theories from MLNs.

For convenience,  in the possibilistic logic theories, we will use {\em typed} formulas. For instance, when we have the formula
$\alpha = \textit{owns}(\textit{person}:X,\textit{thing}:Y)$
and the set of constants of the type {\it person} is $\{\textit{alice}, \textit{bob} \}$ and the set of constants of the type {\it thing} is $\{ \textit{car} \}$ then $\alpha$ corresponds to the ground formulas $\textit{owns}(alice,car)$ and $\textit{owns}(bob,car)$. In cases where there is only one type, we will not write it explicitly.

Two typed formulas $F_1$ and $F_2$ are said to be isomorphic when there is a type-respecting substitution $\theta$ of the variables of $F_1$ such that $F_1\theta \equiv F_2$ (where $\equiv$ denotes equivalence of logical formulas).  Two MLNs $\mathcal{M}_1$ and $\mathcal{M}_2$ are said to be isomorphic, denoted by $\mathcal{M}_1 \approx \mathcal{M}_2$, if there is a bijection $i$ from formulas of $\mathcal{M}_1$ to formulas of $\mathcal{M}_2$ such that for $i(F,w) = (F',w')$ it holds that $w = w'$ and the formulas $F$ and $F'$ are isomorphic. When $j$ is a permutation of a subset of constants from $\mathcal{M}$ then $j(\mathcal{M})$ denotes the MLN obtained by replacing any constant $c$ from the subset by its image $j(c)$. 
 
Given a non-ground MLN $\mathcal{M}$, we can first identify sets of constants which are {\em interchangeable}, where a set of constants $\mathcal{C}_{t}$ is said to be interchangeable if $j(\mathcal{M}) \approx \mathcal{M}$ for any permutation $j$ of the constants in $\mathcal{C}_t$. Note that to check whether a set of constants $\mathcal{C}_t$ is interchangeable, it is sufficient to check that $j(\mathcal{M}) \approx \mathcal{M}$ for those permutations which swap just two constants from $\mathcal{C}_t$. For every maximal set $\mathcal{C}_t$ of interchangeable constants, we introduce a new type $t$. For a constant $c$, we write $\tau(c)$ to denote its type. When $F$ is a ground formula, $\textit{variabilize}(F)$ denotes the following formula:
$$
\bigwedge \{V_c \neq V_d \,|\, c,d\in \textit{const}(F), \tau(c)=\tau(d)\} \rightarrow F'
$$
%$\left(\bigwedge_{\mathcal{C}_t} \bigwedge_{c',c'' \in \mathcal{C}_t \colon c' \neq c''} V_{c'} \neq V_{c''} \right) \rightarrow F'$ 
where $\textit{const}(F)$ is the set of constants appearing in $F$ and $F'$ is obtained from $F$ by replacing all constants $c$ by a new variable $V_c$ of type $\tau(c)$. 
%\todo{It seems that we can use considerably more compact formulas for $\textit{variabilize}(F)$, by only including those constants in the antecedent that actually occur in $F$. Yes, that's what I do in the code too. I did not phrase this very well.} 
%Using the introduced concepts and notations, a transformation of non-ground MLNs into possibilistic logic theories can be described as follows (the main differences with respect to the ground version of the method are bold-faced).
\begin{definition}\label{defTransformationFirstOrder}
Given an MLN $\mathcal{M}$ and a positive integer $k$, we construct a possibilistic logic theory $\Theta^k_{\mathcal{M}}$ as follows:
\begin{itemize}
\item For each hard rule $F$ from $\mathcal{M}$, add $(F,1)$ to $\Theta^k_{\mathcal{M}}$.
\item For each set of literals $E$ such that $0 \leq |E| \leq k$, let $X = \{x \,|\, (\mathcal{M},E) \mapent x\}$. 
{For all $x \in X$}, unless there is a literal $y \in E$ such that $(\mathcal{M}, E \setminus \{y\}) \mapent y$ and unless $\Theta^k_{\mathcal{M}}$ already contains a formula isomorphic to $\textit{variabilize}\left(\bigwedge E \rightarrow \mathbf{x} \right)$, add 
$$\left(\textit{variabilize}\left(  \bigwedge E \rightarrow \mathbf{x} \right), \lambda_E \right) $$
 to $\Theta^k_{\mathcal{M}}$. If $\textit{pen}(\mathcal{M},E) >\textit{pen}(\mathcal{M},\emptyset)$ and  $\Theta^k_{\mathcal{M}}$ does not already contain a formula isomorphic to $\textit{variabilize} \left( \neg \left(\bigwedge E \wedge \bigwedge X \right) \right)$, add also
 \begin{align}\label{eqBlockingRuleDefaultsFirstOrder}
\left(\textit{variabilize} \left( \neg \left(\bigwedge E \wedge \bigwedge X \right) \right), \lambda_E' \right)
\end{align}
where $\lambda_E'$ is the certainty level just below $\lambda_E$ in $\Theta^k_{\mathcal{M}}$.
\end{itemize}
\end{definition}
As before, we will usually omit the subscript in $\Theta^k_{\mathcal{M}}$.
We can show that after grounding, $\Theta^k$ is equivalent to the theory that would be obtained by first grounding the MLN and then applying the method from Section \ref{section:defaultEncoding}. The correctness proof is provided in the online appendix. 
%It is relatively straightforward to see that when fully grounded, the constructed possibilistic logic theory will be equivalent to a possibilistic logic theory that we would obtain using the method described in Section \ref{section:defaultEncoding}.

%Similarly, as ground possibilistic logic theories, the produced non-ground possibilistic logic theories can be postprocessed. Since they is not full first-order theories, but 'merely' templates for constructing finite ground theories, we cannot use skolemization when checking whether a clause is implied by a possibilistic logic theory. However, when we want 

%****************************************************************
%\subsection{Implementation}\label{sec:implementation}
Our implementation\footnote{The implementation can be downloaded from: https://github.com/supertweety/mln2poss.}
 of Transformation \ref{defTransformationFirstOrder} relies on an efficient implementation of inference in possibilistic logic and Markov logic, efficient generation of non-redundant candidate evidence sets and efficient filtering of isomorphic formulas. For MAP inference in MLNs, we used a cutting-plane inference algorithm % based on ideas from \cite{riedel08}, using a SAT-based optimization instead of integer-linear programming.
based on a SAT-based optimization. 
For inference in possibilistic logic, we also used cutting-plane inference in order to avoid having to ground the whole theory. To find the ground rules that need to be added by the cutting-plane method, we used a modified querying system from \cite{kuzelka08}. For solving and optimizing the resulting ground programs, we used the SAT4J library \cite{sat4j}. 

Note that to check whether $\Theta^k_\lambda \vdash F$, where $F$ is a (not necessarily ground) clause, it is sufficient to find one (type-respecting) grounding $\theta$ of $F$, 
%\todo{I think we can remove: using the constants of the respective types and satisfying the inequalities which have been introduced in $F$ by the {\em variabilization} procedure. OK, you're right. Maybe we can say type-respecting substitution to avoid any possibility of confusion?}, 
and check whether $\Theta^k_\lambda \cup \{ \neg (F\theta) \}$ is inconsistent. 
%It follows from the construction of the possibilistic logic theory that this simple method works correctly. 
In this way, we can check whether a rule is implied by $\Theta^k$ without grounding the whole theory because, as for MLNs, inference in non-ground possibilistic logic theories can be carried out by cutting-plane inference methods.

We implemented the transformation as a modification of the standard best-first search (BFS) algorithm which constructs incrementally larger candidate evidence sets, checks their MAP consequences and adds the respective rules to the possibilistic logic theory being constructed. Like the standard BFS algorithm it uses a hash-table based data structure {\it closed}, in which already processed evidence sets are stored. In order to avoid having to check isomorphism with every evidence set in {\it closed}, each time a new evidence set is considered, the stored evidence sets are enriched by fingerprints which contain some invariants, guaranteeing that no two variabilized evidence sets with different fingerprints are isomorphic. In this way, we can efficiently check for a given evidence set $E$ whether there is a previously generated evidence set $E'$ such that $\textit{variabilize}(E)$ and $\textit{variabilize}(E')$ are isomorphic.% In this way we can avoid processing evidence sets which are redundant for the construction of the theory.  %The fingerprints are then also used to compute hashes for storing in the hash tables.

As a final remark, we note that for the non-ground transformation, it may be preferable to replace any rule $\left(\textit{variabilize} \left( \neg \left(\bigwedge E \wedge \bigwedge X \right) \right), \lambda_E' \right)$ by the rule $\left(\textit{variabilize} \left( \neg \bigwedge E \right), \lambda_E' \right)$. The reason is that the former rules may often become too long in the non-ground case. On the other hand, for the ground transformation, the advantage of the longer rules is that they will often be the same for different sets $E$, which, in effect, means a smaller number of rules in the possibilistic logic theory. The correctness of this alternative to Transformation 4 is also shown in the online appendix.

%****************************************************************
\section{ILLUSTRATIVE EXAMPLES}\label{secExamples}

The first example is a variation on a classical problem from non-monotonic reasoning. Here, we want to express that birds generally fly, but heavy antarctic birds do not fly, unless they have a jet pack. The MLN which we will convert into possibilistic logic contains the following rules:
$10 :  \textit{bird}(X) \rightarrow \textit{flies}(X)$, 
$1  : \textit{antarctic}(X) \rightarrow \neg \textit{flies}(X)$,
$10  : \textit{heavy}(X) \rightarrow \neg \textit{flies}(X)$,
$100 : \textit{hasJetPack}(X) \rightarrow \textit{flies}(X)$.
When presented with this MLN, Transformation \ref{defTransformationFirstOrder} produces the following possibilistic logic theory.
\begin{align*}
( \neg \textit{antarctic}(X) \vee  \neg \textit{flies}(X), \lambda_{0}) \\
( \neg \textit{bird}(X) \vee \textit{flies}(X), \lambda_{0}) \\
( \neg \textit{heavy}(X) \vee  \neg \textit{flies}(X), \lambda_{0}) \\
(\textit{flies}(X) \vee  \neg \textit{hasJetPack}(X), \lambda_{0}) \\
( \neg \textit{bird}(X) \vee \textit{flies}(X) \vee \textit{hasJetPack}(X), \lambda_{1}) \\
( \neg \textit{heavy}(X) \vee \textit{antarctic}(X) \vee  \neg \textit{flies}(X), \lambda_{1}) \\
( \neg \textit{bird}(X) \vee  \neg \textit{heavy}(X), \lambda_{1}) \\
( \neg \textit{antarctic}(X) \vee  \neg \textit{heavy}(X) \vee  \neg \textit{flies}(X), \lambda_{10}) \\
(\textit{flies}(X) \vee  \neg \textit{hasJetPack}(X) \vee \textit{bird}(X), \lambda_{11}) \\
( \neg \textit{bird}(X) \vee \textit{flies}(X) \vee  \neg \textit{hasJetPack}(X), \lambda_{100})
\end{align*}
Let us consider the evidence set  $E=\{\textit{bird}(\textit{tweety}),\textit{heavy}(\textit{tweety})\}$. Then the levels $\lambda_0$ and $\lambda_1$ drown because of the inconsistency with the rule $( \neg \textit{bird}(X) \vee  \neg \textit{heavy}(X), \lambda_{1})$ which was produced as one of the rules \eqref{eqBlockingRuleDefaultsFirstOrder}. We can see from the rest of the possibilistic logic theory that unless we add either $\textit{antarctic}(\textit{tweety})$ or $\textit{hasJetPack}(\textit{tweety})$, we cannot say anything about whether {\it tweety} flies or not. It can be verified that the same is true also for the respective MLN.

%\todo{We can say something about the case when the evidence set is bird(tweety) and heavy(tweety), in which case the blocking rule on $\lambda_1$ is activated and from the resulting theory we get formulas from which it follows that if we add that tweety is antarctic then it will not fly or if we add that tweety has a jet-pack then it will fly. }

The second example consists of formulas from a classical MLN about smokers. There are three predicates in this MLN: a binary predicate $\textit{f}(A,B)$ denoting that $A$ and $B$ are friends, and two unary predicates $\textit{s}(A)$ and $\textit{c}(A)$ denoting that $A$ smokes and that $A$ has cancer, respectively. The MLN contains the following hard rules: $\neg \textit{f}(\textit{A}, \textit{B}) \vee \textit{f}(\textit{B}, \textit{A})$ and $ \neg \textit{f}(\textit{A},\textit{A})$. In addition, we have two soft rules. The first soft rule $10 \colon \neg \textit{s}(\textit{A}) \vee  \neg \textit{f}(\textit{A}, \textit{B}) \vee \textit{s}(\textit{B})$ states that if $A$ and $B$ are friends and $A$ smokes then $B$ is more likely to smoke too. The second rule $10 \colon \neg \textit{s}(\textit{A}) \vee \textit{c}(\textit{A})$ states that smoking increases the likelihood of cancer. The following possibilistic logic theory was obtained using Transformation \ref{defTransformationFirstOrder} with $k=4$.
\begin{align*}
(\textit{s}(\textit{B}) \vee  \neg \textit{f}(\textit{A}, \textit{B}) \vee \neg \textit{s}(\textit{A}) \vee  \neg \textit{alldiff}(\textit{A}, \textit{B}), \lambda_{0}) \\
( \neg \textit{s}(\textit{A}) \vee \textit{c}(\textit{A}), \lambda_{0}) \\
%( \neg \textit{f}(\textit{C}, \textit{B}) \vee  \neg \textit{f}(\textit{A}, \textit{B})  \vee  \neg \textit{s}(\textit{A}) \vee  \neg \textit{s}(\textit{C}) \vee \textit{c}(\textit{B}) \quad\quad\quad\quad \\  \vee  \neg \textit{alldiff}(\textit{A}, \textit{B}, \textit{C}), \lambda_{10}) \\
( \neg \textit{f}(\textit{C}, \textit{B}) \vee  \neg \textit{f}(\textit{A}, \textit{B}) \vee \textit{s}(\textit{A}) \vee \textit{s}(\textit{C}) \quad\quad\quad\quad\quad\quad\quad\quad \\  \vee  \neg \textit{alldiff}(\textit{A}, \textit{B}, \textit{C}) \vee  \neg \textit{s}(\textit{B}), \lambda_{10}) \\
( \neg \textit{f}(\textit{C}, \textit{B}) \vee  \neg \textit{s}(\textit{A}) \vee  \neg \textit{f}(\textit{A}, \textit{C}) \vee \textit{s}(\textit{C}) \quad\quad\quad\quad\quad\quad\quad\quad \\ \vee \neg \textit{alldiff}(\textit{A}, \textit{B}, \textit{C}) \vee  \neg \textit{s}(\textit{B}), \lambda_{10}) \\
( \neg \textit{s}(\textit{A}) \vee  \neg \textit{f}(\textit{C}, \textit{A}) \vee \textit{s}(\textit{C}) \vee \textit{c}(\textit{B}) \quad\quad\quad\quad\quad\quad\quad\quad \\ \vee  \neg \textit{alldiff}(\textit{A}, \textit{B}, \textit{C}) \vee  \neg \textit{s}(\textit{B}), \lambda_{10}) \\
%(\textit{s}(\textit{B}) \vee  \neg \textit{f}(\textit{B}, \textit{C}) \vee \textit{c}(\textit{C}) \vee  \neg \textit{s}(\textit{A}) \quad\quad\quad\quad\quad\quad\quad\quad \\  \vee \neg \textit{f}(\textit{A}, \textit{C}) \vee  \neg \textit{alldiff}(\textit{A}, \textit{B}, \textit{C}), \lambda_{10}) \\
( \neg \textit{s}(\textit{A}) \vee \textit{c}(\textit{A}) \vee \textit{c}(\textit{B}) \vee  \neg \textit{s}(\textit{B}) \vee  \neg \textit{alldiff}(\textit{A}, \textit{B}), \lambda_{10}) \\
(\textit{s}(\textit{B}) \vee  \neg \textit{f}(\textit{A}, \textit{B}) \vee  \neg \textit{s}(\textit{A}) \vee \textit{c}(\textit{A}) \vee  \neg \textit{alldiff}(\textit{A}, \textit{B}), \lambda_{10}) \\
( \neg \textit{f}(\textit{A}, \textit{B}) \vee \textit{f}(\textit{B}, \textit{A}), 1) \\
( \neg \textit{f}(\textit{A}, \textit{A}), 1) 
\end{align*}
At the lowest level $\lambda_0$ we find the counterparts of the soft rules from the MLN, whereas at level 1 we find the hard rules. At the intermediate level we intuitively find weakened rules from the MLN. For instance, the rule $( \neg \textit{s}(\textit{A}) \vee \textit{c}(\textit{A}) \vee \textit{c}(\textit{B}) \vee  \neg \textit{s}(\textit{B}) \vee  \neg \textit{alldiff}(\textit{A}, \textit{B}), \lambda_{1})$ can be interpreted as: if $A$ and $B$ smoke then at least one of them has cancer. It is quite natural that this rule has higher certainty weight than the rule: if $A$ smokes then $A$ has cancer.

A final, more elaborate example is provided in the online appendix.

%****************************************************************
\section{RELATED WORK}\label{secRelatedWork}

One line of related work focuses on extracting a comprehensible model from another learned model that is difficult or impossible to interpret. A seminal work in this area is the TREPAN~\cite{craven.nips8} algorithm. Given a trained neural network and a data set, TREPAN learns a decision tree to mimic the predictions of the neural network. In addition to producing interpretable output, this algorithm was shown to learn accurate models that faithfully mimicked the neural network's predictions.  More recent research has focused on approximating complex ensemble classifiers with a single model. For example, Popovic et al.~\cite{popovic:bibm13} proposed a method for learning a single decision tree that mimics the predictions of a random forest. 

While, to the best of our knowledge, this is the first paper that studies the relation between Markov logic and possibilistic logic, the links between possibility theory and probability theory have been widely studied. For example, \cite{probpossTransformation} has proposed a probability-possibility transformation based on the view that a possibility measure corresponds to a particular family of probability measures. Dempster-Shafer evidence theory \cite{shafer1976mathematical} has also been used to provide a probabilistic interpretation to possibility degrees. In particular, a possibility distribution can be interpreted as the contour function of a mass assignment; see \cite{dubois1989fuzzy} for details. In \cite{Saint-cyr94penaltylogic} it is shown how the probability distribution induced by a penalty logic theory corresponds to the contour function of a mass assignment, which suggests that it is indeed natural to interpret this probability distribution as a possibility distribution. Several other links between possibility theory and probability theory have been discussed in \cite{dubois2006possibility}. 

In this paper, we have mainly focused on MAP inference. An interesting question is whether it would be possible to construct a (possibilistic) logic base that captures the set of accepted beliefs encoded by a probability distribution, where $A$ is accepted if $P(A)> P(\neg A)$. Unfortunately, the results in \cite{dubois2004ordinal} show that this is only possible for the limited class of so-called big-stepped probability distributions. In practice, this means that we would have to define a partition of the set of possible worlds, such that the probability distribution over the partition classes is big-stepped, and only capture the beliefs that are encoded by the latter, less informative, probability distribution. A similar approach was taken in \cite{doi:10.1142/S0218488503002235} to learn default rules from data.

%****************************************************************
\section{CONCLUSIONS}
This paper has focused on how a Markov logic network $\mathcal{M}$ can be encoded in possibilistic logic. We started from the observation that it is always possible to construct a possibilistic logic theory $\Theta_{\mathcal{M}}$ that is equivalent to $\mathcal{M}$, in the sense that the probability distribution induced by $\mathcal{M}$ is isomorphic to the possibility distribution induced by $\Theta_{\mathcal{M}}$. As a result, applying possibilistic logic inference to $\Theta_{\mathcal{M}}$ yields the same conclusions as applying MAP inference to $\mathcal{M}$. Although the size of $\Theta_{\mathcal{M}}$ is exponential in the number of formulas in $\mathcal{M}$, we have shown how more compact theories can be obtained in cases where we can put restrictions on the types of evidence that need to be considered (e.g.\ small sets of literals).% or where we only want to derive conclusions about particular literals.

Our main motivation has been to use possibilistic logic as a way to make explicit the assumptions encoded in a given MLN. Among others, the possibilistic logic theory could be used to generate explanations for predictions made by the MLN, to gain insight into the data from which the MLN was learned, or to identify errors in the structure or weights of the MLN. Taking this last idea one step further, our aim for future work is to study methods for repairing a given MLN, based on the mistakes that have thus been identified.

\subsubsection*{Acknowledgements}
We would like to thank the anonymous reviewers for their helpful comments.
This work has been supported by a grant from the Leverhulme Trust (RPG-2014-164). JD is partially supported by the Research Fund KU Leuven (OT/11/051), EU FP7 Marie Curie Career Integration Grant (\#294068) and FWO-Vlaanderen(G.0356.12).

\bibliographystyle{abbrv}
\bibliography{mln,kr}

\ifdefined\WITHSUPPLEMENT

\appendix
\section{PROOFS}

\begin{proof}[Proof of Proposition \ref{prop1}]
Let $X = \{(F,w) \in \mathcal{M} \colon \omega \not\models F \}$ be the set of rules from the MLN not satisfied in $\omega$. By construction, the possibilistic logic theory $\Theta$ contains the rule $(\bigvee X^*, \lambda)$ where $\lambda = \phi(\neg \bigvee X^*) = \frac{K+\textit{pen}(\mathcal{M},\neg \bigvee X^*)}{L} = \frac{K+\textit{pen}(\mathcal{M},\omega)}{L}$. As a result, $\omega \not\models \Theta_\lambda$. On the other hand, $\omega \models \Theta_{\lambda'}$ for any $\lambda' > \lambda$. %which can be shown as follows. 
%If there was a rule $(\bigvee Y^*, \lambda') \in \Theta$ such that $Y \subseteq \mathcal{M}$, $\lambda' > \lambda$ and $\omega \not\models \bigvee Y^*$ then necessarily, for any $\alpha \in Y^*$, it would have to hold $\omega \not\models \alpha$ and as a consequence $\textit{pen}(\mathcal{M},\omega) > \lambda$, which is a contradiction. 
Indeed, for $(\bigvee Y^*, \lambda') \in \Theta$ and $\lambda' > \lambda$, we have by construction that $Y^*\not\subseteq X^*$, i.e.\ $Y^*$ contains a formula $\alpha$ which is not in $X^*$. By definition of $X$, this means $\omega\models \alpha$, and thus in particular $\omega\models \bigvee Y^*$.
Since $\pi$ was assumed to be the least specific model of $\Theta$, it holds that $\pi(\omega) = 1 - \max_{(\alpha_i,\lambda_i)} \{ \lambda_i  | \omega \not\models \alpha_i \}$. Hence we can conclude $\pi(\omega) = 1 - \phi(\omega)$.
\end{proof}

\begin{proof}[Proof of Lemma \ref{lemma:firstlemma}]
The lemma can be proved by induction on $|E|$. The base case $|E| = 0$ is obvious. Let us assume that the lemma holds for $|E| = n < k$. If  $E$ does not contain any literal $y$ such that $(\mathcal{M},E\setminus \{y\}) \mapent y$ then $\Theta^k_\lambda$ contains a formula $\bigwedge E \rightarrow \bigwedge X$ with $x\in X$ and the lemma clearly holds. Otherwise, if there is a literal $y \in E$ such that $(\mathcal{M},E\setminus \{y\}) \mapent y$ then we must also have $(\mathcal{M},E\setminus \{y\}) \mapent x$. Moreover $\phi(E)=\phi(E\setminus \{y\})=\lambda$ in this case. By induction we find that the formula $\bigwedge (E \setminus \{ y \}) \rightarrow x$ can be derived from $\left\{ \left( \bigwedge E' \rightarrow \bigwedge X \right) \in \Theta^k_{\lambda} \mbox{ s.t. } |E'| \leq |E| \right\}$. 
%The lemma can be proved by induction on $|E|$. The base case $|E| = 0$ is obvious. Let us assume that the lemma holds for $|E| = n < k$. If  $E$ does not contain any literal $y$ such that $(\mathcal{M},E\setminus \{y\}) \mapent y$ then $\Theta^k_\lambda$ contains a formula $\bigwedge E \rightarrow \bigwedge X$ with $x\in X$ and the lemma clearly holds. Otherwise, if there is a literal $y \in E$ such that $(\mathcal{M},E\setminus \{y\}) \mapent y$ then we must also have $(\mathcal{M},E\setminus \{y\}) \mapent x$. Since moreover $\phi(E)=\phi(E\setminus \{y\})=\lambda$ in this case, it holds that $\Theta^k_\lambda$ contains a formula $\bigwedge (E\setminus \{y\})\bigwedge X$ such that $x\in X$. By induction we find that the latter formula can be derived from $\left\{ \left( \bigwedge E' \rightarrow \bigwedge X \right) \in \Theta^k_{\lambda} \mbox{ s.t. } |E'| \leq |E| \right\}$. 
\end{proof}

\begin{proof}[Proof of Lemma \ref{LemmaSelectiveMAPStates}]
We have that $\{F_1,...,F_l\} \cap \mathcal{Y} \neq \emptyset$ for every $\mathcal{Y} \in \textit{Cons}_E(\mathcal{M})$ iff $\bigvee \{ \bigwedge \mathcal{Y} \,|\, \mathcal{Y}\in \textit{Cons}_E(\mathcal{M})\} \wedge \neg F_1 \wedge ... \wedge \neg F_l$ is inconsistent. This in turn means that the most plausible world of $E\wedge \neg F_1 \wedge ... \wedge \neg F_l$ cannot be among the most plausible worlds of $E$, and thus $\textit{pen}(\mathcal{M}, E \wedge \neg F_1 \wedge ... \wedge \neg F_l) > \textit{pen}(\mathcal{M}, E)$.
\end{proof}

\begin{proof}[Proof of Lemma \ref{lemma:moving}]
(i) Let $\Omega$ be the set of most probable worlds of $(\mathcal{M}, E)$, let $\Omega'_i$ be the set of most probable worlds of $(\mathcal{M}, E \cup \{ \neg y_i \})$ and let $\Omega''_i$ be the set of most probable worlds of $(\mathcal{M}, E)$ in which $y_i$ is false. In general $\Omega'_i \neq \Omega_i''$. Either $(\mathcal{M}, E) \mapent (y_1 \wedge \dots \wedge y_k)$ or at least one of $\Omega''_i$ must be nonempty. Let $\Omega''_{i^*}$ be such a nonempty set. It must hold $(\mathcal{M} , E \cup \{ \neg y_{i^*} \}) \mapent (y_1 \vee \dots \vee y_{i^*-1} \vee y_{i^*+1} \vee \dots \vee y_k)$ because, in this case, $\Omega_{i^*}' = \Omega_{i^*}'' \subseteq \Omega$. (ii) The second part of the lemma follows from repeated application of the first part.
\end{proof}

\begin{proof}[Proof of Lemma \ref{lemma:clausalconsequence}]
($\Rightarrow$) It follows from Lemma \ref{lemma:lemma2} that any most probable model of $E$ is a model of $\Theta^k_\lambda \cup E$,   hence if $\cP_\lambda \cup E \vdash \bigvee C$ then $(\mathcal{M},E) \mapent \bigvee C$.

($\Leftarrow$) 
If $(\mathcal{M}, E) \mapent (c_{1} \vee \dots \vee c_{m})$, then by Lemma \ref{lemma:moving} we have
$(\mathcal{M}, E \cup  C') \mapent c_{j}$ for some $j\in\{1,...,m\}$ and $C' \subseteq \{ \neg c_{1}, \dots, \neg c_{j-1}, \neg c_{j+1}, \dots, \neg c_{m} \}$. Since $|E|+|C'| \leq k$, by Lemma \ref{lemma:firstlemma} we have $\Theta^k_\lambda \vdash \neg \left( \bigwedge E \right) \vee \neg \left( \bigwedge C'  \right) \vee c_j$ and therefore $\Theta^k_\lambda \cup E \vdash \bigvee C$.
\end{proof}

\subsubsection*{Correctness of Transformation \ref{defTransformationFirstOrder}}
In the following lemmas, $\mathcal{M}$ will be an MLN, $k$ will be an integer,  $\Upsilon^{k}_{\mathcal{M}}$ will be the ground possibilistic logic encoding given by Transformation \ref{transformationDefaults} in which, for convenience, we replace any rule $\left(\bigwedge E \rightarrow \bigwedge X, \lambda_E \right)$ by a set of rules $\left(\bigwedge E \rightarrow x, \lambda_E \right)$ where $x \in X$. $\Theta^k_{\mathcal{M}}$ will be the non-ground possibilistic logic encoding given by Transformation \ref{defTransformationFirstOrder}. Recall that $\Theta^k_\mathcal{M}$ is given together with a set of constants, typed according to their interchangeability. The ground possibilistic logic theory obtained by grounding $\Theta_{\mathcal{M}}^k$ using this set of typed constants will be denoted by $\widehat{\Theta}_{\mathcal{M}}^k$.

\begin{lemma}\label{lemma:nonground1}
Let $\alpha = (\bigwedge E \rightarrow x, \lambda_E)$ be a rule. Then $\alpha \in \Upsilon^k_{\mathcal{M}}$ if and only if $\alpha \in \widehat{\Theta}_{\mathcal{M}}^k.$
\end{lemma}
\begin{proof}
($\Rightarrow$) If $\alpha \in \Upsilon^k_{\mathcal{M}}$ then a rule isomorphic to $(\textit{variabilize}\left(  \bigwedge E \rightarrow x \right), \lambda_E)$ must be contained in $\Theta^k_\mathcal{M}$ but then $\alpha$ must be among the groundings of that rule and therefore also be in $\widehat{\Theta}_{\mathcal{M}}^k$.

($\Leftarrow$) If $\alpha \in \widehat{\Theta}_{\mathcal{M}}^k$ then $\Theta^k_\mathcal{M}$ must contain a rule $\left(\textit{variabilize}\left(\bigwedge E' \rightarrow x'\right), \lambda_E\right)$ isomorphic to $\left(\textit{variabilize}\left(\bigwedge E \rightarrow x\right), \lambda_E\right)$ such that $\left(\mathcal{M},E'\right) \mapent x'$ and such that there is no literal $y' \in E'$ satisfying $(\mathcal{M},E'\setminus \{ y' \}) \mapent y'$. It follows from the fact that constants of the same type are interchangeable that then also $\left(\mathcal{M},E\right) \mapent x$ and that there is no literal $y \in E$ satisfying $(\mathcal{M},E\setminus \{ y \}) \mapent y$. From this it immediately follows that $\alpha \in \Upsilon^k_\mathcal{M}$.
\end{proof}

\begin{lemma}\label{lemma:nonground2}
Let $\alpha = \left( \neg \left(\bigwedge E \wedge \bigwedge X \right), \lambda_E' \right)$ be a rule. Then $\alpha \in \Upsilon^k_{\mathcal{M}}$ if and only if $\alpha \in \widehat{\Theta}_{\mathcal{M}}^k$.
\end{lemma}
\begin{proof}
The proof of this lemma is very similar to the proof of Lemma \ref{lemma:nonground1}.

($\Rightarrow$) If $\alpha \in \Upsilon^k_\mathcal{M}$ then a rule isomorphic to $(\textit{variabilize}\left(  \bigwedge E \wedge \bigwedge X \right), \lambda_E')$ must be contained in $\Theta^k_\mathcal{M}$ but then $\alpha$ must be among the groundings of that rule and therefore also be in $\widehat{\Theta}_{\mathcal{M}}^k$.

($\Leftarrow$) If $\alpha \in \widehat{\Theta}_{\mathcal{M}}^k$ then $\Theta^k_\mathcal{M}$ must contain a rule $\left(\textit{variabilize}\left(\bigwedge E' \wedge \bigwedge X'\right), \lambda_E'\right)$ isomorphic to $\left(\textit{variabilize}\left(\bigwedge E \wedge \bigwedge x\right), \lambda_E'\right)$ such that $\left(\mathcal{M},E'\right) \mapent x'$ for all $x' \in X'$. It follows from the fact that constants of the same type are interchangeable that then also $\left(\mathcal{M},E\right) \mapent x$ for all $x \in X$, from which it immediately follows that $\alpha \in \Upsilon^k_\mathcal{M}$.
\end{proof}

\begin{proposition}\label{firstordercorrectness}
Given an MLN $\mathcal{M}$ and an integer $k$, the possibilistic logic theories obtained by Transformation \ref{transformationDefaults} and by grounding the possibilistic logic theory obtained by Transformation \ref{defTransformationFirstOrder} are equivalent.
\end{proposition}
\begin{proof}
The proof follows from Lemma \ref{lemma:nonground1} and Lemma \ref{lemma:nonground2}.
\end{proof}

\subsubsection*{Correctness of the alternative to Transformation \ref{defTransformationFirstOrder}}
We show the correctness of the alternative transformation, in which we replace rules of the form 
$$
\left(\textit{variabilize} \left( \neg \left(\bigwedge E \wedge \bigwedge X \right) \right), \lambda_E' \right)
$$ 
by the rule
$$
\left(\textit{variabilize} \left( \neg \bigwedge E \right), \lambda_E' \right)
$$
The correctness of this alternative transformation can be shown as follows. Let $\Theta^k$ be the possibilistic logic theory obtained by Transformation \ref{defTransformationFirstOrder} and let $\Phi^k$ be the possibilistic logic theory obtained by the alternative transformation. It holds that $\neg \bigwedge E \models \neg (\bigwedge E \wedge \bigwedge X)$ and consequently also $\textit{variabilize}(\neg \bigwedge E) \models \textit{variabilize}(\neg (\bigwedge E \wedge \bigwedge X))$. Therefore if $\Theta^k_\lambda \cup F$ is inconsistent then $\Phi^k_\lambda \cup F$ must be inconsistent as well. Moreover, since $\lambda_E'< \phi(E)$, we can show (using arguments analogical to those used in the proof of Proposition \ref{prop:defaultbased}) that this alternative transformation also satisfies the properties stated for Transformation \ref{defTransformationFirstOrder} in Proposition~\ref{firstordercorrectness}.

\section{ADDITIONAL EXAMPLE}

\begin{table*}[t]
\centering
  \begin{tabular}{ll} \hline
$1$ & $\neg \textit{wrote}(\textit{per:A},\textit{pap:C}) \vee \neg \textit{wrote}(\textit{per:A},\textit{pap:B}) \vee \textit{category}(\textit{pap:C},\textit{cat:D}) \vee \neg \textit{category}(\textit{pap:B},\textit{cat:D})$ \\
$2$ & $\neg \textit{refers}(\textit{pap:A},\textit{pap:B}) \vee \textit{category}(\textit{pap:B},\textit{cat:C}) \vee \neg \textit{category}(\textit{pap:A},\textit{cat:C})$ \\
$2$ & $\neg \textit{refers}(\textit{pap:A},\textit{pap:B}) \vee \textit{category}(\textit{pap:A},\textit{cat:C}) \vee \neg \textit{category}(\textit{pap:B},\textit{cat:C})$ \\
$-3$ & $\textit{category}(\textit{pap:A},\textit{cat:net})$ \\
$0.14$ & $\textit{category}(\textit{pap:A},\textit{cat:prog})$ \\
$0.09$ & $\textit{category}(\textit{pap:A},\textit{cat:os})$ \\
$0.04$ & $\textit{category}(\textit{pap:A},\textit{cat:hw\_arch})$ \\
$0.11$ & $\textit{category}(\textit{pap:A},\textit{cat:ds\_alg})$ \\
$0.04$ & $\textit{category}(\textit{pap:A},\textit{cat:enc\_compr})$ \\
$0.02$ & $\textit{category}(\textit{pap:A},\textit{cat:ir})$ \\
$0.05$ & $\textit{category}(\textit{pap:A},\textit{cat:db})$ \\
$0.39$ & $\textit{category}(\textit{pap:A},\textit{cat:ai})$ \\
$0.06$ & $\textit{category}(\textit{pap:A},\textit{cat:hci})$ \\
$0.06$ & $\textit{category}(\textit{pap:A},\textit{cat:net})$ \\ 
$\infty$ & $\textit{eq}(\textit{cat:B},\textit{cat:C}) \vee \neg \textit{category}(\textit{pap:A},\textit{cat:C}) \vee \neg \textit{category}(\textit{pap:A},\textit{cat:B})$ \\ \hline
  \end{tabular}
  \caption{Markov Logic Network for CORA.}
  \label{tab:1}
\end{table*}

In this section we illustrate Transformation \ref{defTransformationFirstOrder} on a larger MLN trained for predicting categories of computer science papers\footnote{We obtained the structure of the MLN from the Tuffy web {\tt http://i.stanford.edu/hazy/tuffy/download/}.}, which is shown in Table \ref{tab:1}. This MLN contains rules for predicting categories of papers from categories of other papers which refer to them (rules 2-3) or from categories of papers written by the same author (rule 3). In addition it contains rules giving prior probabilities of the individual probabilities and a hard rule specifying that every paper has at most one category (for simplicity).

We applied Transformation \ref{defTransformationFirstOrder} on this MLN which resulted in a possibilistic logic theory with 200 rules\footnote{There are two modes for filtering redundant rules in the possibilistic logic theory in our implementation. The more aggressive filtering iteratively removes rules which are entailed by other rules in the theory (with the same or greater certainty weights) whereas the less aggressive filtering only iteratively removes the rules which are entailed by subset of the rules from the theory which are all shorter or equally long. The latter filtering results in slightly more interpretable results. In the experiments reported here, we have used the less aggressive filtering.}. Table \ref{tab:2} displays a subset of the rules in the resulting theory, which are interesting for illustrating some properties of the MLN and which are not immediately obvious from the MLN itself. Notice that we represent, apart from the formulas of the form (\ref{eqBlockingRuleDefaultsFirstOrder}), we present formuals in the implication form to make the interpretation easier. 

Rule $\alpha$ is one of the rules enforcing drowning. Rule $\beta$ states that in absence of other evidence, the category of any paper is assumed to be AI. Rule $\gamma$ cannot intuitively be justified and is probably an unintended consequence of the MLN. It states that if $V1$ wrote a paper which is not from the category AI, then $V1$ did not write any other paper. This rule is not harmful when the MLN is used for the purpose of predicting categories but it indicates that the MLN would not be suitable for predicting authorship (which is not really surprising given that the MLN was not trained for this task). The rules $\delta$ and $\eta$ are representatives of rules which capture the prior distribution of the categories (note that there are more rules of this kind which we do not show here). Rule $\zeta$ is similar to $\gamma$. The rules $\iota$, $\kappa$, $\lambda$, $\mu$, $\nu$ and $\xi$ state that if one paper refers to another paper then they typically have the same category. Such rules actually give evidence of meaningfulness of the MLN for prediction of categories.

\begin{table*}[t]
\centering
  \begin{tabular}{l} \hline
$\alpha = ( \neg \textit{category}(\textit{pap:V1}, \textit{cat:prog}), \lambda_{0})$ \\
$\beta = ( \rightarrow \textit{category}(\textit{pap:V1}, \textit{cat:ai}), \lambda_{0})$ \\ \hline
$\gamma = (\textit{wrote}(\textit{per:V1}, \textit{pap:V3}) \wedge \neg \textit{eq}(\textit{pap:V2}, \textit{pap:V3}) \wedge  \neg \textit{category}(\textit{pap:V2}, \textit{cat:ai}) \rightarrow  \neg \textit{wrote}(\textit{per:V1}, \textit{pap:V2}), \lambda_{1})$ \\
$\delta§ = ( \neg \textit{category}(\textit{pap:V1}, \textit{cat:ai}) \rightarrow \textit{category}(\textit{pap:V1}, \textit{cat:prog}), \lambda_{1})$ \\
$\epsilon = ( \neg \textit{category}(\textit{pap:V1}, \textit{cat:ds\_alg\_th}), \lambda_{1})$ \\ \hline
$\zeta = (\textit{category}(\textit{pap:V2}, \textit{cat:ds\_alg}) \wedge \textit{wrote}(\textit{per:V1}, \textit{pap:V3}) \wedge \neg \textit{eq}(\textit{pap:V2}, \textit{pap:V3}) \rightarrow \neg \textit{wrote}(\textit{per:V1}, \textit{pap:V2}), \lambda_{2})$ \\
$\eta = ( \neg \textit{category}(\textit{pap:V1}, \textit{cat:ai}) \wedge  \neg \textit{category}(\textit{pap:V1}, \textit{cat:prog}) \rightarrow \textit{category}(\textit{pap:V1}, \textit{cat:ds\_alg}), \lambda_{2})$ \\
$\theta = ( \neg \textit{category}(\textit{pap:V1}, \textit{cat:os}), \lambda_{2})$ \\ \hline
$\dots$\\ \hline
$\iota = (\neg \textit{eq}(\textit{pap:V1}, \textit{pap:V2}) \wedge \textit{refers}(\textit{pap:V1}, \textit{pap:V2}) \wedge  \neg \textit{category}(\textit{pap:V2}, \textit{cat:ai}) \rightarrow \textit{category}(\textit{pap:V1}, \textit{cat:prog}), \lambda_{8})$ \\
$\kappa = (\neg \textit{eq}(\textit{pap:V1}, \textit{pap:V2}) \wedge \textit{refers}(\textit{pap:V2}, \textit{pap:V1}) \wedge  \neg \textit{category}(\textit{pap:V2}, \textit{cat:ai}) \rightarrow \textit{category}(\textit{pap:V1}, \textit{cat:prog}), \lambda_{8})$ \\
\dots
\\ \hline
\dots \\
\hline 
$\lambda = (\neg \textit{eq}(\textit{pap:V1}, \textit{pap:V2}) \wedge \textit{refers}(\textit{pap:V2}, \textit{pap:V1}) \wedge \textit{category}(\textit{pap:V2}, \textit{cat:ds\_alg}) \rightarrow$ \\ \multicolumn{1}{r}{$\textit{category}(\textit{pap:V1}, \textit{cat:ds\_alg}), \lambda_{11})$} \\
$\mu = (\neg \textit{eq}(\textit{pap:V1}, \textit{pap:V2}) \wedge \textit{refers}(\textit{pap:V2}, \textit{pap:V1}) \wedge \textit{category}(\textit{pap:V1}, \textit{cat:ds\_alg}) \rightarrow $ \\ \multicolumn{1}{r}{$\textit{category}(\textit{pap:V2}, \textit{cat:ds\_alg}), \lambda_{11})$} \\ 
\dots \\ \hline
\dots \\ \hline
$\nu = (\neg \textit{eq}(\textit{pap:V1}, \textit{pap:V2}) \wedge \textit{refers}(\textit{pap:V2}, \textit{pap:V1}) \wedge \textit{category}(\textit{pap:V1}, \textit{cat:os}) \rightarrow \textit{category}(\textit{pap:V2}, \textit{cat:os}), \lambda_{14})$ \\
$\xi = (\neg \textit{eq}(\textit{pap:V1}, \textit{pap:V2}) \wedge \textit{refers}(\textit{pap:V2}, \textit{pap:V1}) \wedge \textit{category}(\textit{pap:V2}, \textit{cat:os}) \rightarrow \textit{category}(\textit{pap:V1}, \textit{cat:os}), \lambda_{14})$ \\
\dots \\ \hline
  \end{tabular}
  \caption{Subset of the possibilistic logic theory obtained for the CORA MLN with maximum evidence set size $k = 2$. The different levels of the theory are separated by horizontal lines.}
  \label{tab:2}
\end{table*}

\fi

\end{document}